\documentclass[runningheads,a4paper,12pt,twoside,oribibl,envcountsame]{article}
\usepackage[margin=1in]{geometry}
\newenvironment{keywords}{\centerline{\bf\small Keywords}\begin{quote}\small}{\par\end{quote}\vskip 1ex}
  \newdimen\paravsp  
\def\paradot#1{\vspace{\paravsp plus 0.5\paravsp minus 0.5\paravsp}\noindent{\bf\boldmath{#1.}}} 
\def\paradoto#1{\vspace{\paravsp plus 0.5\paravsp minus 0.5\paravsp}\noindent{\bf\boldmath{#1.}}} 

\usepackage[mathletters]{ucs} 
\usepackage[utf8x]{inputenc} 
\usepackage{latexsym,amsmath,amssymb,bm}
\usepackage{hyperref} 
\usepackage{color}    
\usepackage{pifont}

\def\,{\mskip 3mu} \def\>{\mskip 4mu plus 2mu minus 4mu} \def\;{\mskip 5mu plus 5mu} \def\!{\mskip-3mu}
\newtheorem{theorem}{Theorem}

\newtheorem{lemma}[theorem]{Lemma}

\newenvironment{proof}{{\noindent\bf Proof.}}{\vskip 1ex}

\def\hrefurl#1{\href{#1}{\rule{0ex}{1.7ex}\color{blue}\underline{\smash{#1}}}} 
\def\nq{\hspace{-1em}}          
\def\fr#1#2{{\textstyle\frac{#1}{#2}}} 
\def\frs#1#2{{^{#1}\!/\!_{#2}}} 
\def\qed{\hspace*{\fill}\rule{1.4ex}{1.4ex}$\quad$\\} 
\def\v{\boldsymbol}             
\def\m#1{{\bf #1}}              
\def\trp{{\!\top\!}}            
\def\cC{\mathcal{C}}
\def\cF{\mathcal{F}}
\def\cG{\mathcal{G}}
\def\cS{\mathcal{S}}
\def\cX{\mathcal{X}}
\def\cY{\mathcal{Y}}
\def\MLP{\text{MLP}}
\def\Poly{\text{Poly}}

\begin{document}

\title{\vspace{-4ex}
\normalsize\sc 
\vskip 2mm\bf\Large\hrule height5pt \vskip 4mm
On Representing (Anti)Symmetric Functions%
\vskip 4mm \hrule height2pt
}
\author{{\bf Marcus Hutter}\\[3mm]
\normalsize DeepMind \& ANU\\[2mm]
\normalsize \hrefurl{http://www.hutter1.net/}}
\date{12 Jun 2020}
\maketitle

\begin{abstract}
Permutation-invariant, -equivariant, and -covariant functions and anti-symmetric functions
are important in quantum physics, computer vision, and other disciplines.
Applications often require most or all of the following properties:
(a) a large class of such functions can be approximated, e.g.\ all continuous function,
(b) \emph{only} the (anti)symmetric functions can be represented,
(c) a fast algorithm for computing the approximation,
(d) the representation itself is continuous or differentiable,
(e) the architecture is suitable for learning the function from data.
(Anti)symmetric neural networks have recently been developed and applied with great success.
A few theoretical approximation results have been proven,
but many questions are still open, 
especially for particles in more than one dimension
and the anti-symmetric case, which this work focusses on.
More concretely, we derive natural polynomial approximations in the symmetric case,
and approximations based on a \emph{single} generalized Slater determinant in the anti-symmetric case.
Unlike some previous super-exponential and discontinuous approximations,
these seem a more promising basis for future tighter bounds.
We provide a complete and explicit
universality proof of the Equivariant MultiLayer Perceptron,
which implies universality of symmetric MLPs and the FermiNet.

\vspace{5ex}\def\contentsname{\centering\normalsize Contents}\setcounter{tocdepth}{1}
{\parskip=-2.7ex\tableofcontents}
\end{abstract}


\section{Introduction}\label{sec:Intro}

\paradoto{Neural networks}
Neural Networks (NN), or more precisely, Multi-Layer Perceptrons (MLP), 
are universal function approximators \cite{Pinkus:99}
in the sense that every (say) continuous function can be approximated arbitrarily well by a sufficiently large NN.
The true power of NN though stems from the fact that they apparently have a bias towards functions we care about
and that they can be trained by local gradient-descent or variations thereof.

\paradoto{Covariant functions}
For many problems we have additional information about the function,
e.g.\ symmetries under which the function of interest is invariant or covariant.
Here we consider functions that are covariant%
\footnote{In full generality, a function $f:\cX→\cY$ is covariant under group operations $g∈G$,
if $f(R_g^X(x))=R_g^Y(f(x))$, where $R_g^X:\cX→\cX$ and $R_g^Y:\cY→\cY$ are representations of group (element) $g∈G$.}
under permutations.\footnote{The symmetric group $G=S_n$ is the group of all the permutations $π:\{1,...,n\}→\{1,...,n\}$.}
Of particular interest are functions that are invariant%
\footnote{$R_g^Y$=Identity. Permutation-invariant functions
are also called `totally symmetric functions' or simply `symmetric function'.}, 
equivariant\footnote{General $\cY$ and $\cX$, often $\cY=\cX$ and $R_g^Y=R_g^X$, also called \emph{covariant}.}, or 
anti-symmetric\footnote{$R_g^Y=±1$ for even/odd permutations.} under permutations.
Of course (anti)symmetric functions are also just functions,
hence a NN of sufficient capacity can also represent (anti)symmetric functions,
and if trained on an (anti)symmetric target could converge to an (anti)symmetric function.
But NNs that can represent \emph{only} (anti)symmetric functions are desirable for multiple reasons.
Equivariant MLP (EMLP) are the basis for constructing symmetric functions by simply summing the output of the last layer,
and for anti-symmetric (AS) functions by multiplying with Vandermonde determinants 
or by computing their generalized Slater determinant (GSD).

\paradoto{Applications}
The most prominent application is in quantum physics which represents systems of identical 
(fermions) bosons with (anti)symmetric wave functions \cite{Pfau:19}.
Another application is classification of point clouds in computer vision,
which should be invariant under permutation of points \cite{Zaheer:18}.

\paradoto{Exact (anti)symmetry}
Even if a general NN can learn the (anti)symmetry, it will only do so approximately,
but some applications require exact (anti)symmetry, 
for instance in quantum physics to guarantee upper bounds on the true ground state energy \cite{Pfau:19}.
This has spawned interest in NNs that can represent \emph{only} (anti)symmetric functions \cite{Zaheer:18,Han:19}.
A natural question is whether such NNs can represent \emph{all} reasonable (anti)symmetric functions,
which is the focus of this paper.
We will answer this question for the (symmetric) EMLP \cite{Zaheer:18}
and for the (AS) FermiNet \cite{Pfau:19}.

\paradoto{Desirable properties}
Approximation architectures need to satisfy a number of criteria to be practically useful:
\begin{itemize}\parskip=0ex\parsep=0ex\itemsep=0ex
\item[(a)] they can approximate a large class of functions, \\
           e.g.\ all continuous (anti)symmetric functions,
\item[(b)] \emph{only} the (anti)symmetric functions can be represented,
\item[(c)] a fast algorithm exists for computing the approximation,
\item[(d)] the representation itself is continuous or differentiable,
\item[(e)] the architecture is suitable for learning the function from data \\
           (which we don't discuss).
\end{itemize}

\paradoto{Content}
Section~\ref{sec:Back} reviews existing approximation results for (anti)symmetric functions.
Section~\ref{sec:oneds} discusses various ``naive'' representations
(linear, sampling, sorting) and their (dis)advantages,
before introducing the ``standard'' solution that satisfies (a)-(e) 
based on algebraic composition of basis functions,
symmetric polynomials, and polarized bases.
For simplicity 
the section considers only totally symmetric functions of their $n$ real-valued inputs (the $d=1$ case), i.e.\ particles in one dimension.
Section~\ref{sec:oneda} proves the representation power of a single GSD for totally anti-symmetric (AS) functions (also $d=1$).
Technically we reduce the GSD to a Vandermonde determinant,
and determine the loss of differentiability due to the Vandermonde determinant.
From Sections~\ref{sec:threed} on we consider the general case of functions with $n⋅d$ inputs 
that are (anti)symmetric when permuting their $n$ $d$-dimensional input vectors.
The case $d=3$ is particularly relevant for particles and point clouds in 3D space.
The difficulties encountered for $d=1$ transfer to $d>1$, 
while the positive results don't, or only with considerable extra effort.
Section~\ref{sec:nn} reviews classical NN approximation theory as a preparation for 
Equivariant MLPs proven universal in Section~\ref{sec:covnn},
which are then used in Section \ref{sec:assnn} to prove universality of 
Symmetric MLPs and of the AS FermiNet.
Section~\ref{sec:Disc} concludes.
A list of notation can be found in Appendix~\ref{app:Notation}.

\paradoto{Remarks}
The equivariance construction in Section~\ref{sec:covnn} is long,
but the result and construction are rather natural.
But to the best of our knowledge no
proof has previously been published.
It may have tacitly been assumed that the universality of the polarized superposition Theorems~\ref{thm:palgbone}\&\ref{thm:palgb}
imply universality of the EMLP.
Section~\ref{sec:oneds} reviews various approaches to representing symmetric functions,
and is the broadest albeit very compact review we are aware of.
We also took care to avoid relying on results with inherently asymptotic or tabulation character,
to enable (in future work) good approximation rates for specific function classes, 
such as those with `nice' Fourier transform \cite{Barron:93,Makovoz:96},
The results about the FermiNet are non-trivial and unexpected.

\section{Related Work}\label{sec:Back}

\paradoto{NN approximation theory \cite{Pinkus:99,Lu:20}}
The study of universal approximation properties of NN has a long history,
see e.g.\ \cite{Pinkus:99} for a pre-millennium survey, 
and e.g.\ \cite{Lu:20} for recent results and references.
For (anti)symmetric NN such investigation has only recently begun
\cite{Zaheer:18,Wagstaff:19,Han:19,Sannai:19}.

\paradoto{Zaher\&al.(2018) \cite{Zaheer:18}}
Functions on sets are necessarily invariant under permutation,
since the order of set elements is irrelevant.
For countable domain, \cite{Zaheer:18} derive a general representation based on encoding domain elements
as bits into the binary expansion of real numbers.
They conjecture that the construction can be generalized to uncountable domains such as $ℝ^d$,
but it would have to involve pathological everywhere discontinuous functions \cite{Wagstaff:19}.
Functions on sets of fixed size $n$ are equivalent to symmetric functions in $n$ variables.
\cite{Zaheer:18} prove a symmetric version of Kolmogorov-Arnold's superposition theorem \cite{Kolmogorov:57}
(for $d=1$) based on elementary symmetric polynomials und using Newton's identities, 
also known as Girard-Newton or Newton-Girard formulae,
which we will generalize to $d>1$.
Another proof is provided based on homeomorphisms between vectors and ordered vectors,
also with no obvious generalization to $d>1$.
They do not consider AS functions.

\paradoto{Han\&al.(2019) \cite{Han:19}}
For symmetric functions and any $d≥1$, 
\cite{Han:19} provide two proofs of the symmetric superposition theorem of \cite{Zaheer:18}:
Every symmetric function can be approximated by symmetric polynomials,
symmetrized monomials can be represented as a permanents,
and Ryser's formula brings the representation into the desired polarized superposition form.
The down-side is that computing permanents is NP complete,
and exponentially many symmetrized monomials are needed to approximate $f$. 
The second proof discretizes the input space into a $n⋅d$-dimensional lattice
and uses indicator functions for each grid cell.
They then symmetrize the indicator functions,
and approximate $f$ by these piecewise constant symmetric indicator functions instead of polynomials,
also using Ryser formula for the final representation.
Super-exponentially many indicator functions are needed,
but explicit error bounds are provided.
The construction is discontinuous but they remark on how to make it continuous.
Approximating AS $f$ for $d≥1$ is based on a similar lattice construction, 
but by summing super-exponentially many Vandermonde determinants, leading to a similar bound.
We show that a single Vandermonde/Slater determinant suffices but without bound. 
Additionally for $d=1$ we determine the loss in smoothness this construction suffers from.

\paradoto{Sannei\&al.(2019) \cite{Sannai:19}}
\cite{Sannai:19} prove tighter but still exponential bounds if $f$ is Lipschitz w.r.t.\ $\ell^∞$
based on sorting which inevitably introduces irreparable discontinuities for $d>1$.

\paradoto{Pfau\&al.(2019) \cite{Pfau:19}}
The FermiNet \cite{Pfau:19} is also based on EMLPs \cite{Zaheer:18}
but anti-symmetrizes not with Vandermonde determinants but with GSDs.
It has shown remarkable practical performance for modelling the ground state of a variety of atoms and small molecules. 
To achieve good performance, a linear combination of GSDs has been used. 
We show that in principle a single GSD suffices, 
a sort of generalized Hartree-Fock approximation. 
This is contrast to the increasing number of conventional Slater determinants required for increasing accuracy.
Our result implies (with some caveats) that the improved practical performance of multiple GSDs
is due to a limited (approximation and/or learning) capacity of the EMLP,
rather than a fundamental limit of the GSD.

\section{One-Dimensional Symmetry}\label{sec:oneds}

This section reviews various approaches to representing symmetric functions,
and is the broadest review we are aware of. 
To ease discussion and notation, we consider $d=1$ in this section,
Most considerations generalize easily to $d>1$, 
some require significant effort, and others break.
We discuss various ``naive'' representations
(linear, sampling, sorting) and their (dis)advantages,
before introducing the ``standard'' solution that can satisfy (a)-(e).
All representations consist of a finite set of fixed (inner) basis functions,
which are linearly, algebraically, functionally, or otherwise combined.
We then provide various examples, including composition by inversion and 
symmetric polynomials, which can be used to prove the ``standard'' representation theorem for $d=1$.

\paradot{Motivation}
Consider $n∈ℕ$ one-dimensional particles with coordinates 
$x_i∈ℝ$ for particle $i=1,...,n$. In quantum mechanics the probability amplitude
of the ground state can be described by a real-valued joint 
wave function $χ(x_1,...,x_n)$. Bosons $ϕ$ have a totally symmetric
wave function: $ϕ(x_1,...,x_n)=ϕ(x_{π(1)},...,x_{π(n)})$
for all permutations $π∈S_n\subset\{1:n\}\to\{1:n\}$. 
Fermions $ψ$ have totally Anti-Symmetric (AS) wave functions: 
$ψ(x_1,...,x_n)=σ(π)ψ(x_{π(1)},...,x_{π(n)})$, where $σ(π)=±1$ is 
the parity or sign of permutation $π$. 
We are interested in representing or 
approximating all and only such (anti)symmetric functions by neural networks.
Abbreviate $\m x≡(x_1,...,x_n)$ and let $S_π(\m x):=(x_{π(1)},...,x_{π(n)})$
be the permuted coordinates.
There is an easy way to (anti)symmetrize any function,
\begin{align}\label{eq:sas}
  ϕ(\m x) ~=~ {1\over n!}\sum_{π∈S_n} χ(S_π(\m x)), ~~~~~ 
  ψ(\m x) ~=~ {1\over n!}\sum_{π∈S_n} σ(π)χ(S_π(\m x))
\end{align}
and any (anti)symmetric function can be represented in this form (proof: use $χ:=ϕ$ or $χ:=ψ$).
If we train a NN $χ:ℝ^n→ℝ$ to approximate some function $f:ℝ^n→ℝ$ to accuracy $ε>0$,
then $ϕ$ ($ψ$) are (anti)symmetric approximations of $f$ to accuracy $ε>0$ too,
provided $f$ itself is (anti)symmetric.
Instead of averaging, the minimum or maximum or median or many other compositions would also work,
but the average has the advantage that smooth $χ$ lead to smooth $ϕ$ and $ψ$,
and more general, preserves many desirable properties such as (Lipschitz/absolute/...) continuity,
($k$-times) differentiability, analyticity, etc.
It possibly has all important desirable property, but one:

\paradot{Time complexity}
The problem with this approach is that it has $n!$ terms, 
and evaluating $χ$ super-exponentially often is intractable even for moderate $n$,
especially if $χ$ is a NN.
There can also be no clever trick to linearly (anti)symmetrize arbitrary functions fast,
intuitively since the sum pools $n!$ independent regions of $χ$. 
Formally, consider the NP hard Travelling Salesman Problem (TSP):
Let $χ(\m x)=1$ if $x_i=π(i)∀i$ for some $π$ \emph{and} there is a path of length $≤L$
connecting cities in order $π(1)→π(2)→...→π(n)→π(1)$, and $χ(\m x)=0$ 
for all other $\m x$. 
Then $ϕ(1,2,...,n)>0$ iff there exists
a path of length $≤L$ connecting the cities in \emph{some} order. 
Hence $ϕ$ solves the TSP, so $ϕ\!\not\in$P unless P=NP.
For anti-symmetry, define $χ(\m x)=σ(π)$ if ...
The same argument works for min/max/median/...

\paradot{Sampling}
We could sample $O(1/ε^2)$ permutations to approximate $ϕ$ and potentially $ψ$ to accuracy $ε$,
but even 10'000 samples for 1\% accuracy is expensive,
and cancellations of positive and negative terms may require orders of magnitude more samples.
Furthermore, exactly (anti)symmetric wave functions are needed in quantum physics applications.
Even if sampling were competitive for evaluation, 
there may be a super-exponential representation (learning) problem:

\paradot{Learning}
The domain of an (anti)symmetric function consist of $n!$ identical regions. 
The function is the same (apart from sign) on all these regions.  
A general NN must represent the function separately on all these regions,
hence potentially requires $n!$ more samples than an intrinsically (anti)symmetric NN, 
unless the NN architecture and learning algorithm are powerful enough 
to discover the symmetry by themselves and merge all $n!$ regions onto the same internal representation. 
Maybe a NN trained to sort \cite{Wang:95} exhibits this property.
We are not aware of any general investigation of this idea.

\paradot{Function composition and bases}
Before delving into proving universality of the EMLP and the FermiNet,
it is instructive to first review the general concepts of function composition and basis functions,
since a NN essentially is a composition of basis functions.
We want to represent/decompose functions as $f(\m x)=g(\v{β}(\m x))$.
In this work we are interested in symmetric $\v{β}$,
where ultimately $\v{β}$ will be represented by the first (couple of) layer(s) of an EMLP, 
and $g$ by the second (couple of) layer(s).
Of particular interest is 
\begin{align}\label{eq:sumbasis}
  \v{β}(\m x)=\sum_{i=1}^n \v{η}(x_i)
\end{align}
for then $\v{β}$ and hence $f$ are obviously symmetric (permutation invariant) in $\m x$.
Anti-symmetry is more difficult and will be dealt with later.
Formally let $f∈\cF⊆ℝ^n→ℝ$ be a function (class) we wish to represent or approximate.
Let $β_b:ℝ^n→ℝ$ be basis functions for $b=1,...,m∈ℕ∪\{∞\}$,
and $\v{β}≡(β_1,...,β_m):ℝ^n→ℝ^m$ be what we call basis vector (function),
and $η_b:ℝ→ℝ$ a basis template, 
sometimes called inner function \cite{Actor:18} or polarized bass function.
Let $g∈\cG⊆ℝ^m→ℝ$ be a composition function (class),
sometimes called `outer function' \cite{Actor:18},
which creates new functions from the basis functions.
Let $\cG\circ\v{β}=\{g(\v{β}(⋅)):g∈\cG\}$ be the class of representable functions,
and $\overline{\cG\circ\v{β}}$ its topological closure, 
i.e.\ the class of all approximable functions.%
\footnote{Functions may be defined on sub-spaces of $ℝ^k$, 
function composition may not exists,
and convergence can be w.r.t.\ different topologies. 
We will ignore these technicalities unless important for our results,
but the reader may assume compact-open topology, 
which induces uniform convergence on compacta.}
$\v{β}$ is called a $\cG$-basis for $\cF$ if $\cF=\cG\circ\v{β}$ or $\cF=\overline{\cG\circ\v{β}}$,
depending on context.
Interesting classes of compositions are 
linear $\cG_{lin}:=\{g:g(\m x)=a_0+\sum_{i=1}^m x_i;~a_0,a_i∈ℝ\}$,
algebraic $\cG_{alg}:=\{\text{multivariate polynomials}\}$,
functional $\cG_{func}:=ℝ^m→ℝ$,
and $\cC^k$-functional $\cG_{func}^k:=\cC^k$ for $k$-times continuously differentiable functions.

\paradot{Examples}
For $n=1$, $\v{β}(x)=β_1(x)=x$ is an algebraic basis of all polynomials $x$,
and $\overline{\cG_{alg}\circ\v{β}}$ even includes all continuous functions $\cC^0$,
since every continuous function can be approximated arbitrarily well by polynomials.
For $n=1$, $β_b(x)=x^b$ forms a linear basis for all polynomials of degree $m<∞$
and includes all $\cC^0$ functions via closure for $m=∞$.
$β_b(x)=x^{2b-1}$ for $m=∞$ is a linear basis for all continuous axis-AS functions.
For $n=3$, $\v{β}(\m x)=β_1\m (\m x)=x^2+y^2+z^2$ is a functional basis for all rotationally-invariant functions.
For $n=2$, the two elementary symmetric polynomials $\v{β}(\m x)=(x_1+x_2,x_1 x_2)$ 
constitute an algebraic basis for all symmetric polynomials, 
which already requires a bit of work to prove.
Since $2x_1 x_2= (x_1+x_2)^2 - (x_1^2+x_2^2)$, also $\v{β}(\m x)=(x_1+x_2,x_1^2+x_2^2)$ 
is an algebraic basis for all symmetric polynomials, and its closure includes all symmetric continuous functions.
This last basis is of the desired sum form with $η_1(x)=x$ and $η_2(x)=x^2$, 
and has generalizations to $n>2$ and $d>1$ discussed later.
The examples above illustrate that larger composition classes $\cG$ allow (drastically) smaller bases ($m$)
to represent the same functions $\cF$. On the other hand, as we will see, 
algebraic bases can be harder to construct than linear bases.

\paradot{Composition by inversion}
Any injective $\v{β}$ is a functional basis for all functions:
For any $f$, $g(\v w):=f(\v{β}^{-1}(\v w))$ with 
$\v w∈\text{Image}(\v{β})$ represents $f$ as $f(\m x)=g(\v{β}(\m x))$.
If $\v{β}:ℝ→\text{Image}(\v{β})$ is a homeomorphism (diffeomorphism),
then it is a continuous (differentiable) function basis for all continuous (differentiable) functions,
and similarly for other general functions classes.

\paradot{Generally invariant linear bases}
Consider now functions $\cF_\cS$ invariant under some symmetry $\cS⊆ℝ^n→ℝ^n$,
where $\cS$ must be closed under composition,
i.e.\ $\cF_\cS=\{f∈\cF: f(\m x)=f(S(\m x))~\forall S∈\cS,\m x∈ℝ^n\}$ for some $\cF$.
If $\v{β}$ is a linear basis for $\cF$,
then for finite (compact) $\cS$,
$\v{β}_\cS(\m x):=\sum_{S\in\cS} \v{β}(S(\m x))$ 
($\v{β}_\cS(\m x):=\int_\cS \v{β}(S(\m x))dS$) is a linear basis for $\cF_\cS$,
not necessarily minimal.
Above we mentioned the class of rotations $\cS=O(3)$ and rotation-invariant functions/bases.
Symmetrized monomials are discussed below.

\paradot{Symmetric functions by sorting}
Our prime interest is the symmetry class of permutations
$\cS_n:=\{S_π:π∈S_n\}$, where $S_π(x_1,...,x_n)=(x_{π(1)},...,x_{π(n)})$.
An easy functional basis for symmetric functions is $β_b(\m x):=\m x_{[i]}$, 
where $\m x_{[i]}$ is the $i$-th smallest value among $x_1,...,x_n$ (also called order statistics),
i.e.\ 
\begin{align*}
  \min\{x_1,...,x_n\}=\m x_{[1]}≤\m x_{[2]}≤...≤\m x_{[n-1]}≤\m x_{[n]}=\max\{x_1,...,x_n\}
\end{align*}
Obviously $\v{β}(\m x)$ is symmetric (it just sorts its arguments), 
and any symmetric function $ϕ$ can
be represented as 
\begin{align*}
  ϕ(\m x) ~=~ g(\v{β}(\m x)) ~=~ ϕ(\m x_{[1]},...,\m x_{[n]}), ~~~\text{e.g.\ by choosing}~~~ g(\m x)=ϕ(\m x)
\end{align*}
Note that $\v{β}$ is continuous, hence continuous $ϕ$ have a continuous representation,
but $\v{β}$ is not differentiable whenever $x_i=x_j$ for some $i≠j$, 
so smooth $ϕ$ have only non-smooth sorting representations.
Still this is a popular representation for 1-dimensional quantum particles.
This construction still works in higher dimensions, 
but leads to discontinuous functions, as we will see. 

\paradot{Linear basis for symmetric polynomials}
The infinite class of monomials $β_{\v b}(\m x)=x_1^{b_1}⋅...⋅x_n^{b_n}$ for $\v b∈ℕ_0^n$
are a linear basis of all polynomials in $n$ variables.
Hence the symmetrized monomials $β_{\v b}^{S_n}(\m x)=\sum_{π∈S_n} x_{π(1)}^{b_1}⋅...⋅x_{π(n)}^{b_n}$
form a linear basis for all symmetric polynomials,
and by closure includes all continuous symmetric functions.%
This does \emph{not} work for algebraic bases:
While $\v{β}(\m x)=(x_1,...,x_n)$ is an algebraic basis of all polynomials and by closure of $\cC^0$,
$β_b^{S_n}(\m x)=(n-1)!(x_1+...+x_n)$ algebraically generates
only (ridge) functions of the form $g(x_1+...+x_n)$ 
which are constant in all directions orthogonal to $(1,...1)$,
so this is not a viable path for constructing algebraic bases.

\paradot{Algebraic basis for symmetric polynomials}
It is well-known that 
the elementary symmetric polynomials $e_b(\m x)$ generated by
\begin{align}\label{eq:sympol}
  \prod_{i=1}^n (1+λx_i) ~=:~ 1+λe_1(\m x)+λ^2e_2(\m x)+...+λ^n e_n(\m x)
\end{align}
are an algebraic basis of all symmetric polynomials.
Explicit expressions are $e_1(\m x)=\sum_i x_i$, and 
$e_2(\m x)=\sum_{i<j} x_i x_j$, ..., and $e_n(\m x)=x_1...x_n$,
and in general $e_b(\m x)=\sum_{i_1<...<i_b}x_{i_1}...x_{i_b}$.
For given $\m x$, the polynomial in $λ$ on the l.h.s.\ of \eqref{eq:sympol}
can be expanded to the r.h.s.\ in quadratic time or by FFT even in time $O(n\log n)$,
so the $\v e(\m x)$ can be computed in time $O(n\log n)$,
but is not of the desired form \eqref{eq:sumbasis}.
Luckily Newton already solved this problem for us.
Newton's identities express the elementary symmetric polynomials $e_1(\m x),...,e_n(\m x)$ 
as polynomials in $p_b(\m x):=\sum_{i=1}^n x_i^b$, $b=1,...,n$,
hence also $\v{β}(\m x):=(p_1(\m x),...,p_n(\m x))$ is an algebraic basis 
for all symmetric polynomials, 
hence by closure for all continuous symmetric functions,
and is of desired form \eqref{eq:sumbasis}:

\begin{theorem}[Symmetric polarized superposition \text{\cite[Thm.7]{Zaheer:18,Wagstaff:19}}]\label{thm:palgbone}
  Every continuous symmetric function $ϕ:ℝ^n→ℝ$ can be represented 
  as $ϕ(\m x)=g(\sum_i\v{η}(x_i))$ with $\v{η}(x)=(x,x^2,...,x^n)$ and continuous $g:ℝ^n→ℝ$.
\end{theorem}
\cite{Zaheer:18} provide two proofs, one based on `composition by inversion',
the other using symmetric polynomials and Newton's identities.
The non-trivial generalization to $d>1$ is provided in Section~\ref{sec:threed}.

Theorem~\ref{thm:palgbone} is a symmetric version of the infamous Kolmogorov-Arnold superposition theorem \cite{Kolmogorov:57},
which solved Hilbert's 13th problem. Its deep and obscure\footnote{involving continuous $η$ 
with derivative 0 almost everywhere, and not differentiable on a
dense set of points.} constructions continue to fill whole PhD theses \cite{Liu:15,Actor:18}.
It is quite remarkable that the symmetric version above is very natural and comparably easy to prove.

For given $\m x$, the basis $\v{β}(\m x)$ can be computed in time $O(n^2)$, 
so is actually slower than to compute than $\v e(\m x)$.
The elementary symmetric polynomials also have other advantages 
(integral coefficients for integral polynomials, 
works for fields other than $ℝ$, 
is numerically more stable, mimics 1,2,3,... particle interactions),
so symmetric NN based on $e_b$ rather than $p_b$ may be worth pursuing. 
Note that we need at least $m≥n$ functional bases for a \emph{continuous} representation,
so Theorem~\ref{thm:palgbone} is optimal in this sense \cite{Wagstaff:19}.

Table~\ref{tab:bases} summarizes the bases and properties discussed in this section and beyond.
\begin{table}
\caption[Bases and properties]%
{\textbf{\boldmath Bases and properties for $d=1$. Last column comments on $d>1$.} 
Many different representations for symmetric functions have been suggested.
The representations are very heterogenous, so this table is our best but limited attempt to unify and press them into one table.
The table is for $d=1$, but when (some aspects of) the method generalizes to $d>1$ 
we provide \#Bases for $d≥1$ and a comment or reference in the last column.
Most representations can be viewed as instantiations of $ϕ(\m x) ~=~ g(\v{β}(\m x))$ 
with outer function $g$ composing inner base functions $\v{β}$ possibly polarized as $\v{β}=\sum_{i=1}^n\v{η}$.
See main text and references for details, and glossary below. 
The table is roughly in order as described in Section~\ref{sec:oneds}.
For the last three rows, see references. The meaning of the columns is described in the main text.
}\label{tab:bases}\vspace{2ex}
\setlength\tabcolsep{0pt}
\begin{tabular}{|c|c|c|c|c|c|c|c|c|c|}\hline
  Base               & $O()$ comp.        & \#Bases   & Basis      &                 &                            & Prop.       & $\v{β}=?$          & Refs               &        \\
  $\v{β}|\v{η}|$else & time $\v{β}(\m x)$ & $m$       & type$\cG$  & $\cG\circ\v{β}$ & $\overline{\cG\circ\v{β}}$ & of $\v{β}$  & $\sum\v{η}$?       &                    & $d>1$  \\ \hline
  $Σ_π$              & $n!$               & $\aleph_2$& Id         & SymFct          &  --                        & --          & --                 & Sec.\ref{sec:oneds}& same   \\
  Inversion          & depends            & $n$       & $\cC^0$    & $\cC^0$         & =                          & $\cC^0$     & \checkmark         & \cite{Zaheer:18}   & \begin{tabular}{c}un-\\[-1ex]known\end{tabular} \\ 
  $β_b(\m x)=\m x_{[b]}$ & $n\log n$      & $n$       & $\cC^0$    & Sym$\cC^0$      & Sym$\cC^0$                 & $\cC^0$     & no                 & Sec.\ref{sec:oneds}& $\v{β}\!\not\in\cC^0$ \\ 
  \begin{tabular}{c}SymMonom.\\[-1ex]$≤$ deg $D$\end{tabular} & $D^n$ & $({D+n\atop n})$ & Lin & SymPoly$^D$ & =             & $\cC^∞$     & no                 &                    & same \\
  \begin{tabular}{c}AllSym\\[-1ex]Monomials\end{tabular} & $\aleph_0$ & $\aleph_0$ & Lin & SymPoly & Sym$\cC^0$   & $\cC^∞$     & no                 &                    & same \\
  $e_b(\m x)$        & $n\log n$ & $({n+d\atop d})-1$ & Alg        & SymPoly         & Sym$\cC^0$Fct              & $\cC^∞$     & no                 & Sec.\ref{sec:oneds}& Sec.\ref{sec:threed} \\ 
  $η_b(x)=x^b$       & $n^2$     & $({n+d\atop d})-1$ & Alg        & SymPoly         & Sym$\cC^0$Fct              & $\cC^∞$     & \checkmark         & \cite{Zaheer:18}   & Sec.\ref{sec:threed} \\ 
  $η_b=$Cantor       & $n\log n$          & $1$       & All        & AllSymFct       & \multicolumn{2}{c|}{~--~$|$total.discont.} & \checkmark     & \cite{Zaheer:18}   & same \\
  $ε$-Grid           & $(\frs1{ε})^{dn}$  & $(\frs1{ε})^{dn}$ & Lin & SymPCG         & \multicolumn{2}{c|}{~=~$|∞$-indicator} & \checkmark         & \cite{Lu:20}       & same \\                        
  \begin{tabular}{c}Smoothed\\[-1ex]$ε$-Grid\end{tabular} & $(\frs1{ε})^{dn}$  & $(\frs1{ε})^{dn}$ & Lin & \begin{tabular}{c}Smoothed\\[-1ex]SymPC\end{tabular} & \multicolumn{2}{c|}{~$\cC^0$~$|∞$-indicator} & \checkmark    & \cite{Lu:20}       & same \\                        
\hline
\end{tabular}
\\
\begin{tabbing}
  \hspace{0.16\textwidth} \= \hspace{0.70\textwidth} \= \kill
  {\it Symbol }      \> {\it Explanation}                                                    \\[0.5ex]
  $d,n∈ℕ$            \> dimensionality,number of particles                                   \\[0.5ex]
  $\m x∈ℝ^n$         \> $\m x=(x_1,...,x_n)$, function argument, NN input, $n$ 1d particles ($d=1$) \\[0.5ex]
  $b∈\{1:m\}$        \> index of basis function                                              \\[0.5ex]
  $\v{β}:ℝ^n→ℝ^m$    \> $m$ symmetric basis functions                                        \\[0.5ex]
  $\v{η}:ℝ→ℝ^m$      \> $m$ polarized basis functions                                        \\[0.5ex]
  $g:ℝ^m→ℝ$          \> composition or outer function, used as $g(\v{β}(\m x))$              \\[0.5ex]
  Sym(PCG)           \> symmetric (piecewise constant grid)                                  \\[0.5ex]
  Poly$^{(D) }$      \> Multivariate polynomial (of degree at most $D$)                      \\[0.5ex]
  $\cG_{type}\ni g$  \> Composition class (Id-entity, Lin-ear, Alg-ebraic, $\cC^0=$cont., All fcts)\\[0.5ex]
\end{tabbing}
\end{table}

\section{One-Dimensional AntiSymmetry}\label{sec:oneda}

We now consider the anti-symmetric (AS) case for $d=1$.
We provide representations of AS functions
in terms of generalized Slater determinants (GSD) of \emph{partially} symmetric functions.
In later sections we will discuss how these partially symmetric functions
arise from equivariant functions and how to represent equivariant functions by EMLP.
The reason for deferral is that EMLP are inherently tied to $d>1$.
Technically we show that the GSD can be reduced to a Vandermonde determinant,
and exhibit a potential loss of differentiability due to the Vandermonde determinant.

\paradot{Analytic Anti-Symmetry}
Let $φ_i:ℝ→ℝ$ be single-particle wave functions.
Consider the matrix 
\begin{align*}
  Φ(\m x) ~=~
  \left(\begin{matrix}
    φ_1(x_1) & \cdots & φ_n(x_1) \\
     \vdots  & \ddots & \vdots \\
    φ_1(x_n) & \cdots & φ_n(x_n)
  \end{matrix}\right)
\end{align*}
where $\m x≡(x_1,...,x_n)$.
The (Slater) determinant $\det Φ(\m x)$ is anti-symmetric,
but can represent only a small class of AS functions,
essentially the AS analogue of product (wave) functions (pure states, Hartree-Fock approximation).
Every continuous AS function can be approximated/represented by a finite/infinite linear combination of such determinants:
\begin{align*}
  ψ(x_1,...,x_n) ~=~ \sum_{k=1}^∞ \det Φ^{(k)}(\m x), ~~\text{where}~~ Φ^{(k)}_{ij}(\m x):=φ_i^{(k)}(x_j)
\end{align*}
An alternative is to generalize the Slater determinant itself \cite{Pfau:19}
by allowing the functions $φ_i(x)$ to depend on all variables
\begin{align*}
  Φ(\m x) ~=~
  \left(\begin{matrix}
      φ_1(x_1|x_{≠1}) & \cdots & φ_n(x_1|x_{≠1}) \\
          \vdots      & \ddots &    \vdots       \\
      φ_1(x_n|x_{≠n}) & \cdots & φ_n(x_n|x_{≠n})
  \end{matrix}\right)
\end{align*}
where $x_{≠i}≡(x_1,...,x_{i-1},x_{i+1},...,x_n)$.     
If $φ_i(x_j|x_{≠j})$ is symmetric in $x_{≠j}$, which we henceforth assume%
\footnote{The bar $|$ is used to visually indicate this symmetry, otherwise there is no difference to using a comma.},
then exchanging $x_i\leftrightarrow x_j$ is (still) equivalent to exchanging rows $i$ and $j$ in $Φ(\m x)$,
hence $\det Φ$ is still AS. 
The question arises how many GSD are needed to be able to represent \emph{every} AS function $ψ$.
The answer turns out to be `just one', but with non-obvious smoothness relations:
Any AS $ψ$ can be represented by some $Φ$, 
any analytic $ψ$ can be represented by an analytic $Φ$.
The case of continuous(ly differentiable) $ψ$ is more complicated.

\begin{theorem}[Representation of all (analytic) AS $\boldsymbol{ψ}$]\label{thm:as}
   For every (analytic) AS function $ψ(\m x)$ there exist (analytic) 
   $φ_i(x_j|x_{≠j})$ symmetric in $x_{≠j}$ such that $ψ(\m x)=\det Φ(\m x)$.
\end{theorem}

\begin{proof}
Let $φ_1(x_j|x_{≠j}):=χ(x_{1:n})$ be totally symmetric in all $(x_1,...,x_n)$ to be determined later.
Let $φ_i(x_j|x_{≠j}):=x_j^{i-1}$ for $1<i≤d$. Then
\begin{align*}
  \det Φ(\m x) =
  \left|\begin{matrix}
    χ(x_{1:n}) &   x_1  & \cdots & x_1^{n-1} \\
       \vdots  & \vdots & \ddots & \vdots    \\
    χ(x_{1:n}) &   x_n  & \cdots & x_n^{n-1}
  \end{matrix}\right| = χ(x_{1:n})⋅
  \left|\begin{matrix}
       1    &   x_1  & \cdots & x_1^{n-1} \\
    \vdots  & \vdots & \ddots & \vdots    \\
       1    &   x_n  & \cdots & x_n^{n-1}
  \end{matrix}\right| 
  = χ(x_{1:n})\nq\prod_{1≤j<i≤n}\nq(x_i-x_j)
\end{align*}
where the (second) last expression is (the expression for) the Vandermonde determinant.
Since $ψ$ is AS, $ψ(\m x)=0$ if $x_i=x_j$ for any $i≠j$,
hence $ψ$ has factors $x_i-x_j$ for all $i≠j$. Therefore
$χ(x_{1:n}):=ψ(x_{1:n})/\prod_{j<i}(x_i-x_j)$ is totally symmetric,
since $\prod_{j<i}(x_i-x_j)$ is AS,
and obviously $\det Φ(\m x)=ψ(\m x)$ for this choice.
For general AS $ψ$ and $x_i=x_j$ we can define $χ(\m x)$ arbitrarily,
as long as it is symmetric, e.g.\ $χ(\m x)=0$ will do.
For analytic $ψ$, the next lemma shows that $χ$ has an analytic extension.
\qed\end{proof}

\begin{lemma}[\boldmath Symmetric $ψ/∆$ is analytic if AS $ψ$ is analytic]\label{lem:analytic}
Let $ψ:ℝ^n→ℝ$ be AS and analytic. Then $χ(\m x):=ψ(\m x)/∆(\m x)$,
where $∆(\m x):=\prod_{1≤j<i≤n}(x_i-x_j)$,
is totally symmetric and analytic on $ℝ^n\setminus\{\m x:∆(\m x)=0\}$ and 
has a symmetric analytic continuation to all $ℝ^n$.
\end{lemma}

\begin{proof}
Since $ψ$ is analytic, it has a multivariate Taylor series expansion.
For $\v k=(k_1,...,k_n)∈ℕ_0^n$ and $\m x^{\v k}:=x_1^{k_1}⋅⋅⋅x_n^{k_n}$,
we have $ψ(\m x)=\sum_{\v k}a_{\v k}\m x^{\v k}$ for some $a_{\v k}∈ℝ$.
If we anti-symmetrize both sides by $AS[\m x^{\v k}]:=\sum_{π∈S_n}σ(π)x_{π(1)}...x_{π(n)}$,
we get $ψ(\m x)=\sum_{\v k}a_{\v k}AS[\m x^{\v k}]$. Every AS polynomial $AS[\m x^{\v k}]$
can be represented as $AS[\m x^{\v k}]=∆(\m x)S_{\v k}(\m x)$ for some symmetric polynomials $S_{\v k}$.
This follows by successively dividing out all factors $(x_j-x_i)$ \cite[Sec.II.2]{Weyl:46}.
Hence $ψ(\m x)=∆(\m x)χ(\m x)$ with $χ(\m x):=\sum_{\v k}a_{\v k}S_{\v k}(\m x)$,
which is obviously symmetric and analytic.
\qed\end{proof}

\paradot{Continuous/differentiable AS}
We can weaken the analyticity condition as follows:

\begin{theorem}[Representation of continuous/differentiable $\boldsymbol{ψ}$]\label{thm:asdiff}
   Let $\cC^k(ℝ^n)$ be the $k$-times continuously differentiable functions ($k∈ℕ_0$).%
\footnote{For $\v k=(k_1,...,k_n)$, $\cC^{\v k}$ means $∂_{x_1^{k_1}}⋅⋅⋅∂_{x_n^{k_n}}$ exists and is continuous.
$\cC^{|\v k|}:=⋂_{|\v k|=k}\cC^{\v k}$, where $|\v k|:=k_1+...+k_n$. 
We actually only need $ψ∈⋂_{|\v k|=k}\cC^{\v k+(0,1,...,n-1)}\supsetneq \cC^{k+n(n-1)/2}$.}
   For AS function $ψ∈\cC^{k+n(n+1)/2}(ℝ^n)$ there exist 
   $φ_i(x_j|x_{≠j})∈\cC^k(ℝ^n)$ symmetric in $x_{≠j}$ such that $ψ(\m x)=\det Φ(\m x)$.
\end{theorem}

This is a much weaker result than for \emph{linear} anti-symmetrization \eqref{eq:sas}, 
where all $ψ∈\cC^k$ could be represented by $χ∈\cC^k$,
in particular continuous $ψ$ had continuous representations.
For instance, Theorem~\ref{thm:asdiff} (only) implies that $\fr12 n(n+1)$-times differentiable $ψ$ 
have continuous representations, but leaves open whether less than $\fr12 n(n+1)$-times differentiable $ψ$ 
\emph{may} only have discontinuous representations.
It turns out that this is not the case. In Section~\ref{sec:threed} we show
that continuous $ψ$ can be represented by continuous $Φ$,
but whether $ψ∈\cC^k$ has representations with $Φ∈\cC^k$ is open for $k>0$.

\begin{proof}
Consider functions $ψ_A:ℝ^n→ℝ$ with $ψ_A(\m x)=0$ if $x_i=x_j$ for some $(i,j)∈A⊆\{(i,j):1≤i<j≤n\}=:P$.
Note that $ψ_P:=ψ$ satisfies this condition, but the constructed $ψ_A$ will \emph{not} be AS for $A≠P$.
The proof recursively divides out factors $x_j-x_i$: For $A=A'\dot{∪}\{(i,j)\}$ define
\begin{align*}
  ψ_{A'}(\m x) & ~:= \left\{ { {ψ_A(\m x)\over x_j-x_i} ~~~~~~~~~~\text{if}~~~ x_j≠x_i 
                         \atop {∂ψ_A(\m x)\over ∂x_j}|_{x_j=x_i} ~~\text{if}~~~ x_j=x_i } \right.
\end{align*}
If $ψ_A∈\cC^k$ then $ψ_{A'}∈\cC^{k-1}$ by the Lemma~\ref{lem:fxg} below. 
The recursive elimination is independent of the order in which we choose pairs from $A$.
Though not needed, note also that the definition is symmetric in $i\leftrightarrow j$,
since $∂_j ψ|_{x_j=x_i}=-∂_i ψ|_{x_i=x_j}$ due to $ψ(x,x)≡0$ implying $0=dψ(x,x)/dx=∂_1 ψ+∂_2 ψ$.
For $x_j≠x_i$ we obviously have $ψ_{A'}(\m x)=0$ if $x_\jmath=x_\imath$ for some $(\imath,\jmath)∈A'$.
$ψ_{A'}(\m x)=0$ also holds for $x_j=x_i$ by a continuity argument or direct calculation.
We hence can recursively divide out $x_j-x_i$ (in any order)
\begin{align*}
  χ(\m x) ~:=~ ψ_{\{\}}(\m x) ~=...=~ {ψ_A(\m x)\over\prod_{(i,j)∈A}(x_j-x_i)} ~=...=~ {ψ_P(\m x)\over\prod_{(i,j)∈P}(x_j-x_i)} ~=~ {ψ(\m x)\over ∆(\m x)}
\end{align*}
and $χ∈\cC^k$ if $ψ∈\cC^{k+|P|}$. Since $ψ$ and $∆$ are AS, $χ$ is symmetric.
This shows that $φ_1:=χ∈\cC^k$ and the other $φ_i=x_j^{i-1}$ are even analytic.\qed
\end{proof}

Unfortunately this construction does not generalize to $d>1$ dimensions,
but a different construction in Section~\ref{sec:threed} will give a (somewhat) weaker result.
The proof above used the following lemma: 
\begin{lemma}[\boldmath $f∈\cC^k$ implies $f/x∈\cC^{k-1}$]\label{lem:fxg}
  Let $f:ℝ→ℝ$ be $k≥1$-times differentiable and $f(0)=0$, 
  then $g(x):=f(x)/x$ for $x≠0$ and $g(0):=f'(0)$ is 
  $k-1$-times continuously differentiable, i.e.\ $g∈\cC^{k-1}(ℝ)$.
\end{lemma}
The lemma can be proven 
by equating the remainder of the Taylor series expansions of $f$ up to term $k$ with that of $g$.
Only showing continuity of $g^{(k-1)}$ requires some work.
Note that we neither require $f^{(k)}$ to be continuous, nor $f(-x)=-f(x)$ or so.

\section{$d$-dimensional (Anti)Symmetry}\label{sec:threed}

This section generalizes the theorems from Sections~\ref{sec:oneds} and \ref{sec:oneda} to $d>1$:
the symmetric polynomial algebraic basis and the generalized Slater determinant representation.

\paradot{Motivation}
We now consider $n∈ℕ$, $d$-dimensional particles with coordinates 
$\v x_i∈ℝ^d$ for particles $i=1,...,n$.
For $d=3$ we write $\v x_i=(x_i,y_i,z_i)^\trp ∈ℝ^3$.
As before, Bosons/Fermions have symmetric/AS
wave functions $χ(\v x_1,...,\v x_n)$.
That is, $χ$ does not change/changes sign under the exchange of two vectors 
$\v x_i$ and $\v x_j$.
It is \emph{not} symmetric/AS under the exchange of individual coordinates e.g.\ $y_i\leftrightarrow y_j$.
$\m X≡(\v x_1,...,\v x_n)$ is a matrix with $n$ columns and $d$ rows.
The (representation of the) symmetry group is $\cS_n^d:=\{S_π^d:π∈S_n\}$ with
$S_π^d(\v x_1,...,\v x_n):=(\v x_{π(1)},...,\v x_{π(n)})$,
rather than $\cS_{n⋅ d}$. 
Functions $f:ℝ^{d⋅n}→ℝ$ invariant under $\cS_n^d$ are sometimes called multisymmetric or block-symmetric,
if calling them symmetric could cause confusion.

\paradot{Algebraic basis for multisymmetric polynomials}
The elementary symmetric polynomials \eqref{eq:sympol} have a generalization to $d>1$ \cite{Weyl:46}.
We only present them for $d=3$. 
The general case is obvious from them.
They can be generated from
\begin{align}\label{eq:sympold}
  \prod_{i=1}^n (1+λx_i+μy_i+νz_i) ~=:~ \sum_{0≤p+q+r≤n}λ^pμ^qν^r e_{pqr}(\m X)
\end{align}
Even for $d=3$ the expressions are rather cumbersome:
\begin{align*}
  e_{pqr}(\m X) &=~ \sum_{\nq\nq{1≤i_1<...<i_p≤n\atop{1≤j_1<...<j_q≤n\atop 1≤k_1<...<k_r≤n }},\text{all}≠\nq\nq\nq\nq}
  x_{i_1}...x_{i_p}y_{j_1}...y_{j_q}z_{k_1}...z_{k_r} \\
  &=~ {1\over p!q!r!}\sum_{π∈S_n}x_{π(1)}...x_{π(p)}y_{π(p+1)}...y_{π(p+q)}z_{π(p+q+1)}...z_{π(p+q+r)}
\end{align*}
One can show that $\{e_{pqr}:p+q+r≤n\}$ is an algebraic basis 
of size $m=({n+3\atop 3})-1$ for all multisymmetric polynomials \cite{Weyl:46}.
Note that constant $e_{000}$ is not included/needed.
For a given $\m X$, their values can be computed in time $O(m n)$ by expanding
\eqref{eq:sympold} or in time $O(m\log n)$ by FFT, where $m=O(n^d)$.
Newton's identities also generalize:
$e_{pqr}(\m X)$ are 
polynomials in the polarized sums $p_{pqr}(\m X):=\sum_{i=1}^n η_{pqr}(\v x_i)$ 
with $η_{pqr}(\v x):=x^p y^q z^r$. The proofs are much more involved than for $d=1$. 
For the general $d$-case we have:
\begin{theorem}[Multisymmetric polynomial algebraic basis]\label{thm:palgb}
  Every continuous (block\-=multi)\-symmetric function $ϕ:ℝ^{n⋅ d}→ℝ$ can be represented 
  as $ϕ(\m X)=g(\sum_{i=1}^n \v{η}(\v x_i))$ with continuous $g:ℝ^m→ℝ$ and $\v{η}:ℝ^d→ℝ^m$ defined as 
  $η_{p_1...p_d}(\v x)=x^{p_1} y^{p_2}...z^{p_d}$ for $1≤p_1+...+p_d≤n$ ($p_i∈\{0,...,n\}$),
  hence $m=({n+d\atop d})-1$.
\end{theorem}
The basis can be computed in time $O(m⋅n⋅d)=O(d⋅n^{d+1})$ for $n\gg d$ ($=O(nd^{n+1})$ for $d\gg n$).
Note that there could be a much smaller functional bases of size $m=dn$ as per 
``our'' composition-by-inversion argument for continuous representations
in Section~\ref{sec:oneds}, which readily generalizes to $d>1$,
while the above minimal algebraic basis has larger size $m=O(n^d)$ for $d>1$.
It is an open question whether a continuous functional basis of size $O(d n)$ exists,
whether in polarized form \eqref{eq:sumbasis} or not.

\paradot{Anti-Symmetry}
As in Section~\ref{sec:oneda}, consider $Φ_{ij}(\m X):=φ_i(\v x_j|\v x_{≠j})$,
\begin{align*}
  Φ(\m X) ~=~
  \left(\begin{matrix}
      φ_1(\v x_1|\v x_{≠1}) & \cdots & φ_n(\v x_1|\v x_{≠1}) \\
          \vdots      & \ddots &    \vdots       \\
      φ_1(\v x_n|\v x_{≠n}) & \cdots & φ_n(\v x_n|\v x_{≠n})
  \end{matrix}\right)
\end{align*}
where $φ_i(\v x_j|\v x_{≠j})$ is symmetric in $\v x_{≠j}$.
We can show a similar representation result as in Theorem~\ref{thm:as},
but weaker and via a different construction.

\begin{theorem}[Representation of all AS $ψ$]\label{thm:asd}
   For every AS function $ψ(\m X)$ there exist 
   $φ_i(\v x_j|\v x_{≠j})$ symmetric in $\v x_{≠j}$ such that $ψ(\m X)=\det Φ(\m X)$.
\end{theorem}

\begin{proof}
Define any total order $<$ on $ℝ^d$. For definiteness choose ``lexicographical'' order
$\v x_i<\v x_j$ iff $x_i<x_j$ or ($x_i=x_j$ and $y_i<y_j$) or ($x_i=x_j$ and $y_i=y_j$ and $z_i<z_j$),
etc.\ for $d>3$.
Let $\bar{π}∈S_n$ be the
permutation which sorts the particles $\v x_i∈ℝ^d$ in increasing order,
i.e.\ $\v x_{\bar{π}(1)}≤\v x_{\bar{π}(2)}≤...≤\v x_{\bar{π}(n)}$. 
Note that this permutation depends on 
$\m X$, but since $\m X$ is held fixed for the whole proof, we don't need to worry about this.
Now temporarily assume $ψ(\m X)≥0$, and define
\begin{align}\label{eq:asdphi}
  φ_i(\v x_j|\v x_{≠j}) ~:=~ \left\{ 
   \begin{matrix}
     ψ(\v x_{\bar{π}(1)},...,\v x_{\bar{π}(n)})^{1/n} & \text{if} & j=\bar{π}(i) \\
     0                                    & \text{else} &      
   \end{matrix} \right.
\end{align}
which is symmetric in $\v x_{≠j}$. 
If the $\v x_i$ are already sorted, i.e.\ $\bar{π}(i)=i~\forall i$,
then $φ_i(\v x_j|\v x_{≠j})=0$ unless $j=i$.
Hence $Φ(\m X)$ is diagonal with 
\begin{align*}
  \det Φ(\v x) ~=~
  \left|\begin{matrix}
    φ_1(\v x_1|\v x_{≠1})  & \cdots & 0 \\
       \vdots     & \ddots & \vdots    \\
         0   & \cdots & φ_n(\v x_n|\v x_{≠n})
  \end{matrix}\right| 
  ~=~ \prod_{i=1}^n φ_i(\v x_i|\v x_{≠i}) 
  ~=~ ψ(\m X)
\end{align*}
For a general permutation $\bar{π}$, $Φ$ is a permuted diagonal matrix
with only row $\bar{π}(i)$ being non-zero in column $i$:
\begin{align*}
  \det Φ(\v x) &=
  \left|\begin{matrix}
    \vdots  & \cdots & \vdots \\
    \vdots  & \ddots & φ_n(\v x_{\bar{π}(n)}|\v x_{≠\bar{π}(n)}) \\
     φ_1(\v x_{\bar{π}(1)}|\v x_{≠\bar{π}(1)})   & \ddots & \vdots    \\
      \vdots  & \cdots &    \vdots 
  \end{matrix}\right|  
  ~=~ σ(\bar{π})\prod_{i=1}^n φ_i(\v x_{\bar{π}(i)}|\v x_{≠\bar{π}(i)}) \\
  &=~ σ(\bar{π}) ψ(\v x_{\bar{π}(1)},...,\v x_{\bar{π}(n)})
  ~=~ σ(\bar{π})σ(\bar{π}) ψ(\v x_1,...,\v x_n)
  ~=~ ψ(\m X)
\end{align*} 
where we exploited that $ψ$ is AS.  
If $n$ is odd, this construction also works for negative $ψ$.
In general we can replace $ψ^{1/n}$ in (\ref{eq:asdphi}) by
$\text{sign(ψ)}|ψ|^{1/n}$ in $φ_1$ and by $|ψ|^{1/n}$ for the other $φ_i$,
or even $ψ$ for $φ_1$ and $1$ for the other $φ_i$.
\qed\end{proof}
Note that for $d=1$, $(x_{\bar{π}(1)},...,x_{\bar{π}(n)})$ is continuous 
in $(x_1,...,x_n)$, and $\bar{π}$ is (only) discontinuous when $x_i=x_j$ for some $i≠j$,
but then $ψ=0$, hence $φ_i$ in (\ref{eq:asdphi}) is continuous for
continuous $ψ$. Unfortunately this is no longer true for $d>1$.
For instance for $ψ(\v x_1,\v x_2)=y_1-y_2$, 
$φ_1({x_1\atop 1},{x_2\atop 0})=[\![x_1<x_2]\!]$ is a step function.

For $n=2$ and any $d$, any AS continuous/smooth/analytic/$\cC^k$ function $ψ(\v x_1,\v x_2)$ has an easy 
continuous/smooth/analytic/$\cC^k$ representation as a GSD.
Choose $φ_1(\v x_1|\v x_2)≡\fr12$ and $φ_2(\v x_1|\v x_2):=ψ(\v x_1,\v x_2)$.
Whether this generalizes to $n>2$ and $d>1$ is an open problem.

\section{Universality of Neural Networks}\label{sec:nn}

This section is mainly definitions and a few known Neural Network (NN) approximation results,
for setting the stage of the next section for Equivariant MLP.
We present the two key theorems for polynomial-based proofs of universality of NN,
namely that MLP can approximate any polynomial, which in turn can approximate any continuous function.
The NN approximation theory literature is vast, 
so these two results and a few key references must do.

\paradot{Multi-Layer Perceptron (MLP)}
The standard Multi-Layer Perceptron (MLP) aims at approximately representing 
functions $f:ℝ^n→ℝ^m$ ($n,m∈ℕ$, most often $m=1$) as follows:
Let $σ:ℝ→ℝ$ be some non-polynomial\footnote{For Deep networks this can be relaxed to non-linear.} 
continuous activation function.
If not mentioned explicitly otherwise, we will assume that $σ(z)=\max\{0,z\}$
or $σ(z)=\tanh(z)$, and results are valid for both. 
Most choices will also do in practice and in theory,
though there can be some differences. We use function vectorization notation 
$\v{σ}(\v x)_i:=σ(x_i)$. 

The NN input is $\v x∈ℝ^n$, its output is $\v y∈ℝ^m$. 
We use upper indices $(\ell)$ and $(\ell+1)$ without braces for indexing layers. They are never exponents.
Layer $\ell+1$ is computed from layer $\ell∈\{0,...,L-1\}$ by
\begin{align}\label{eq:nn}
  \v x^{\ell+1} ~:=~ \v{τ}^\ell(\v x^\ell) 
  ~≡~ \v{τ}_{\m W^\ell,\v u^\ell}(\v x^\ell)
  ~:=~ \v{σ}(\m W^\ell\v x^\ell+\v u^\ell)
\end{align}
where $\v{τ}^\ell:ℝ^{n_\ell}→ℝ^{n_{\ell+1}}$ is the transfer function 
with synaptic weight matrix $\m W^\ell∈ℝ^{n_{\ell+1}×n_\ell}$ 
from layer $\ell$ to layer $\ell+1$, 
and $\v u^\ell∈ℝ^{n_{\ell+1}}$ the biases of neurons in layer $\ell+1$,
and $n_\ell$ the width of layer $\ell$.

The NN input is $\v x=\v x^0∈ℝ^n$ ($n=n_0$) and the output of an $L$-layer NN is
$\v y=\v x^L∈ℝ^m$ ($m=n_L$).
The NN at layer $\ell+1$ computes the function $\v{ν}^{\ell+1}(\v x):=\v{τ}^\ell(\v{ν}^\ell(\v x))$
with $\v{ν}^0(\v x):=\v x$ and $\v y:=\v{ν}^L(\v x)$.
If $σ$ has bounded range, functions $f$ of greater range cannot be represented.
For this reason, the non-linearity $σ$ in the last layer is often removed,
which we also assume when relevant, and call this an $L-1$-hidden layer NN.

Any $σ$ which is continuously differentiable at least in a small region of its domain,
which is virtually all $σ$ used in practice, 
can approximate arbitrarily well linear hidden neurons and hence skip-connections
by choosing very small/large weights \cite{Pinkus:99}, 
which we occasionally exploit.

\paradot{Literature on approximation results}
There are a large number of NN representation theorems which tell us
how well which functions $f$ can be approximated, 
for function classes of different smoothness (e.g.\ Lipschitz continuous or $\cC^k$),
domain ($ℝ^n$ or a compact subset, typically $[0;1]^n$),
for different distance measures (e.g.\ $L^p$-norm for $p∈[1;∞]$ or Sobolev),
asymptotic or finite bounds in terms of accuracy $ε$, 
network width $N:=\max\{n_0,...,n_L\}$ or depth $L$,
esp.\ deep ($L\gg N=O(1)$) vs.\ shallow ($N\gg L=O(1)$),
or number of neurons $n_+=n_1+...+n_L$,
or total number of weights and biases $\sum_{\ell=0}^{L-1}(n_\ell+1)n_{\ell+1}$,
or only counting the non-zero parameters,
different activation functions $σ$, 
and that are only some choices within this most simple MLP NN model.
See e.g.\ \cite{Yarotsky:17,Lu:20,Grohs:19intro,Rolnick:18,Lin:17} 
and for surveys \cite{Pinkus:99,Fan:19}, to name a few.
A powerful general approximation theory has been developed in \cite{Grohs:19intro},
which very crudely interpreted shows that any function that can approximately be computed
can also be computed by a NN with roughly comparable resources,
which includes even fractal functions such as the Weierstrass function.
Since NN can efficiently emulate Boolean circuits, 
this may not be too surprising.
For simplicity we focus on the most-easy-to-state results,
but briefly mention and reference improvements or variations.

\paradot{Approximation results}
We say that $ρ:ℝ^n→ℝ^m$ \emph{uniformly approximates} $f$ on $[-D;D]^n$ for some $D>0$, 
or $ρ$ \emph{$ε$-approximates} $f$, if 
\begin{align*}
  ||f-ρ||_∞ ~:=~ \sup_{\v x∈[-D;D]^n}\max_{1≤i≤n}|f_i(\v x)-ρ_i(\v x)| ~≤~ ε
\end{align*}
Other norms, most notably $p$-norms and Sobolev norms have been considered \cite{Pinkus:99}.
Adaptation to other compact subsets of $ℝ^n$ of most results is straightforward,
but results on all of $ℝ^n$ are rarer/weaker.
We say that $f$ can be \emph{approximated by a function class} $\cF⊆ℝ^n→ℝ^m$
if for every $ε>0$ and $D$ there exists a $ρ∈\cF$ that 
$ε$-approximates $f$ on (compact) hypercube $[-D;D]^n$,
which is called convergence uniform on compacta.
This is equivalent to $f$ being in the topological closure of $\overline\cF$ 
w.r.t.\ the compact-open topology.
The class $\cF$ we are interested in here
is the class of MLPs defined above and subsets thereof:
\begin{align*}
  \MLP ~:=~ \{ \v{ν}^L ~:~ \m W^\ell∈...,\v u^\ell∈...; n_\ell≤N,\ell≤L; n,m,L,N∈ℕ \} ~⊆~ ℝ^*→ℝ^*
\end{align*}
Note that continuous functions on compact domains have compact range,
which can be embedded in a hyper-cube,
and finite compositions of continuous functions are continuous,
hence each layer as well as the whole MLP maps $[-D;D]^n$ 
continuously into $[-D';D']^m$ for some $D'$.
The two most important results for us are
\begin{theorem}[NN approximation \cite{Pinkus:99}]\label{thm:polymlp}
  Every multivariate polynomial can be approximated by a (1-hidden-layer) MLP.
\end{theorem}
Convergence is uniform on compacta, but also holds w.r.t.\ many other metrics and domains.
For $σ=\max\{0,⋅\}$ one can show that the depth $L$ of the required network grows only 
with $O(\ln ε^{-1})$ and width $N$ is constant.
\cite{Grohs:19intro} shows $N=16$ and $L=O(d+\ln ε^{-1})$ for univariate polynomials of degree $d$,
which easily generalizes to the multivariate case.
Even better, for any twice continuously differentiable non-linear $σ$ such as $\tanh$,
the size of the network does not even need to grow for $ε→0$.
Multiplication $x_1⋅x_2$ can arbitrarily well be approximated by 4 neurons
\cite{Lin:17}, which then allows to compute any given polynomial of degree $d$ in $n$ variables 
consisting of $m$ monomials to arbitrary precision
by a NN of size $O(m⋅n⋅\ln d)$ independent $ε$ \cite{Rolnick:18}.

The other result is the classical Stone-Weierstrass approximation theorem:
\begin{theorem}[Stone-Weierstrass]\label{thm:fctpoly}
  Every continuous function $f:ℝ^n→ℝ^m$ can be approximated by $m$ (multivariate) polynomials.
\end{theorem}
Again convergence is uniform on compacta, but many extensions are known.
Together with Theorem~\ref{thm:polymlp} this implies that MLPs can approximate any continuous function.
For analytic $f$ on $[-D;D]^n$ and $σ=\max\{0,⋅\}$ (and likely most other $σ$),
a NN of depth $L=O(\ln ε^{-1})$ and width $N=O(1)$ suffices \cite{Grohs:19intro}.
For Lipschitz $f$ on $[-D;D]^n$ and $σ=\max\{0,⋅\}$,
a NN of size $N⋅L=\tilde O(m⋅ε^{-n/2})$ suffices \cite[Table 1]{Lu:20}.

Both theorems together show that NN can approximate most functions well,
and the discussion hints at, and \cite{Grohs:19intro} shows, with essentially 
optimal scaling in accuracy $ε$.

\section{Universal Equivariant Networks}\label{sec:covnn}

In this section we will restrict the representation power of MLPs 
to equivariant functions and prove their universality by construction in 4 steps.
This is the penultimate step towards (anti)symmetric NN.

\paradot{Equivariance and all-but-one symmetry} 
We are mostly interested in (anti)symmetric functions,
but for (de)composition we need equivariant functions,
and directly need to consider the $d$-dimensional case.
A function $\v{φ}:(ℝ^d)^n→(ℝ^{d'})^n$ is called equivariant under permutations
if $\v{φ}(S_π^d(\m X))=S_π^{d'}(\v{φ}(\m X))$ for all permutations $π∈S_n$.
With slight abuse of notation we identify $(\v{φ}(\m X))_1≡φ_1(\m X)≡φ_1(\v x_1,\v x_2,...,\v x_n)≡φ_1(\v x_1,\v x_{≠1})$ 
with $φ_1(\v x_1|\v x_{≠1})$. 
The following key Lemma shows that $φ_1$ suffices to describe all of $\v{φ}$. 

\begin{lemma}[All-but-one symmetry = equivariance]\label{lem:cov}
  A function $\v{φ}:(ℝ^d)^n→(ℝ^{d'})^n$ is equivariant (under permutations) if and only if 
  $φ_i(\m X)=φ_1(\v x_i|\v x_{≠i})$ $∀i$ and $φ_1(\v x_i|\v x_{≠i})$ is symmetric in $\v x_{≠i}$. 
\end{lemma}

\begin{proof}
($\Leftarrow$) 
$φ_i(S_π^d(\m X))=φ_1(\v x_{π(i)}|\v x_{π(≠i)})=φ_1(\v x_{π(i)}|\v x_{≠π(i)})=φ_{π(i)}(\m X)=(S_π^{d'}(\v{φ}(\m X)))_i$,
where we used the abbreviation $π(≠i)=(π(1),...,π(i-1),π(i+1),...,π(n))$.
Note that $(≠π(i))=(1,...,π(i-1),π(i+1),...,π(n))$ is the ordered index set, 
and equality holds by assumption on $φ_1$. \\
($\Rightarrow$) Assume $π(i)=i$, then 
$φ_1(\v x_i|\v x_{π(≠1)})=φ_1(S_π^d(\m X))=(S_π^{d'}(\v{φ}(\m X)))_1=φ_1(\m X)$, i.e.\ $φ_1$ is symmetric in $\v x_{≠i}$.
Now assume $π(1)=i$, then $φ_i(\m X)=(S_π^{d'}(\v{φ}(\m X)))_1=φ_1(S_π^d(\m X))=φ_1(\v x_i|\v x_{π(2)},...\v x_{π(n)})=φ_1(\v x_i|\v x_{≠i})$.
\qed\end{proof}

\paradot{Equivariant Neural Network} 
We aim at approximating equivariant $\v{φ}$ by an Equivariant MLP (EMLP).
The NN input is $\m X∈ℝ^{d×n}$,
its output is $\m Y∈ℝ^{d'×n}$. 
Note that $d=3$ for a 3-dimensional physical $n$-particle system ($\v x_i=(x_i,y_i,z_i)^\trp$),
but $d'$ could be a ``feature'' ``vector'' of any length.
We don't aim at modelling other within-vector symmetries, 
such as rotation or translation.

Layer $\ell+1∈\{1,...,L\}$ of EMLP is computed from layer $\ell$ by
\begin{align}\label{eq:covnn}
  \v x_i^{\ell+1} ~:=~ τ_i^\ell(\m X^\ell) 
  ~:=~ τ_1^\ell(\v x_i^\ell|\v x_{≠i}^\ell) 
  ~≡~ τ_{1,\m W^\ell,\m V^\ell,\v u^\ell}(\v x_i^\ell|\v x_{≠i}^\ell)
  ~:=~ \v{σ}(\m W^\ell\v x_i^\ell+\m V^\ell\sum_{j≠i}\v x_j^\ell+\v u^\ell)
\end{align}
where $\v{τ}^\ell:ℝ^{{d_\ell}×n}→ℝ^{d_{\ell+1}×n}$ can be shown to be an equivariant ``transfer'' function 
with weight matrices $\m W^\ell,\m V^\ell∈ℝ^{d_{\ell+1}×d_\ell}$ 
from layer $\ell$ to layer $\ell+1$, 
and $\v u^\ell∈ℝ^{d_{\ell+1}}$ the biases of layer $\ell+1$,
and $d_\ell×n$ is the width of layer $\ell$, now best viewed as a 2-dimensional array 
as in (convolutional) NN used in computer vision.%
\footnote{The formulation $\sum_{j≠i}$ is from \cite{Pfau:19} which is slightly more convenient for our purpose
but equivalent to unrestricted sum $\sum_j$ used by \cite{Zaheer:18}.}

The EMLP input is $\m X=\m X^0∈ℝ^{d×n}$ ($d=d_0$) and the output of an $L$-layer EMLP is
$\m Y=\m X^L∈ℝ^{d'×n}$ ($d'=d_L$).
The NN at layer $\ell+1$ computes the function $\v{ν}^{\ell+1}(\m X):=\v{τ}^\ell(\v{ν}^\ell(\m X))$
with $\v{ν}^0(\m X):=\m X$ and $\v{ν}^L(\m X)=\m Y$.

Inspecting the argument of $\v{σ}()$ shows that $τ_1$ is obviously invariant under
permutation of $\v x_{≠i}$. It is also easy to see that the argument of $\v{σ}()$ 
is the only linear function in $\m X$ with such invariance \cite[Lem.3]{Zaheer:18}.
Lemma~\ref{lem:cov} implies that $\v{τ}$ is equivariant.
Note that if we allowed $\m W$ or $\m V$ or $\v b$ to depend on $i$, this would no longer be the case.
In other words, a single vector-valued function, symmetric in all-but-one vector,
suffices to define any equivariant (matrix-valued) function. 
This tying of weights akin to convolutional networks reduces the number of weights by a factor $≈n/2$. 
Since composition of equivariant functions is equivariant, 
$\v{ν}^L$ is equivariant, i.e.\ the network above can \emph{only} represent equivariant functions.

The question we next answer is whether 
EMLPs can approximate \emph{all} continuous equivariant functions.
We show this in 4 steps: (1) representation of polynomials in a single vector $\v x_i$,
(2) symmetric polynomials in all-but-one vector, (3) equivariant polynomials, 
(4) equivariant continuous functions.

\paradot{Polynomials in a single vector $\v x_i$}
If we set $\m V^\ell=0$ in \eqref{eq:covnn} we get 
$\v x_i^{\ell+1}=\v{σ}(\m W^\ell\v x_i^\ell+\v u^\ell)$,
which is one layer of a general MLP \eqref{eq:nn} in $\v x_i$.
Note that $\v{τ}_{1,\m W^\ell,\m 0,\v u^\ell}$ and hence $\v{ν}^L[\m V^\ell=0∀\ell]$ 
compute the same function and independently for each $\v x_i$.
This can be interpreted as a factored MLP with $n$ identical factors,
or $n$ independent identical MLPs each applied to one $\v x_i$, 
or as one MLP applied to $n$ different vectors $\v x_i$.

Let $ρ_1:ℝ^{d''}→ℝ^{d'''}$ be one such function computed by an EMLP 
with $\m V^\ell=0∀\ell$ to be specified later.
That is, the factored NN computes $\v{ρ}(\m X)=ρ_1(\v x_i)_{i=1}^n$.

We need another factored ($\m V^\ell=0∀\ell$) NN computing 
$η_1:ℝ^d→ℝ^{d''}$, the multivariate polarized basis for $n-1$ 
(because of symmetry in $n-1$ variables only)
$d$-dimensional vectors
$η_1(\v x)=(x^{p_1} y^{p_2}...z^{p_d})_{1≤p_1+...+p_d≤n-1}$, 
where $d''=({n-1+d\atop d})-1$ (cf.\ Theorem~\ref{thm:palgb}).

\paradot{Symmetric polynomials in all-but-one vector}
If we concatenated $\v{η}(\m X)=η_1(\v x_i)_{i=1}^n$ with $\v{ρ}$ we would get 
$\v{ρ}(\v{η}(\v X))$, which is not what we want,
but if we swap $\m V^0=0$ with $\m W^0$ in the NN for $\v{ρ}$ and call it $\v{\tilde ρ}$,
it uses $\sum_{j≠i}\v x_j$ instead of $\v x_i$ as input, so 
we get
\begin{align}\label{eq:emlpc}
  ϕ_1(\v x_{≠i}) ~&:=~ \tilde{ρ}_1(\v{η}(\v X)) ~=~ ρ_1(\v{β}(\m X)),\\
  \text{where}~~~ \v{β}(\m X) &= β_1(\v x_{≠i})_{i=1}^n ~~~ \text{and}~~~
  β_1(\v x_{≠i}):=\sum_{j≠i}η_1(\v x_j)
\end{align}
is the multivariate polarized basis excluding $\v x_i$.
Now we know from Theorem~\ref{thm:palgb} that any polynomial symmetric in 
$\v x_{≠i}$ can be represented as such $ϕ_1$ for suitable polynomial $ρ_1$.
hence approximated by two concatenated EMLPs. By Lemma~\ref{lem:cov}, 
$\v{ρ}$, $\v{\tilde ρ}$, $\v{β}$, $\v{η}$ are all equivariant,
hence also $\v{ϕ}(\m X):=ϕ_1(\v x_{≠i})_{i=1}^n$ is, 
but the latter is not completely general (cf.\ Lemma~\ref{lem:cov}),
which we address now.

\paradot{Equivariant polynomials}
Next we construct polynomials $φ_1(\v x_i|\v x_{≠i})$ symmetric in $\v x_{≠i}$.
Any polynomial in $n$ vectors can be written as
a finite sum over $\v p≡(p_1,...,p_d)∈P⊂ℕ_0^d$ with $|P|<∞$:
\begin{align}\label{eq:etapoly}
  φ_1(\v x_i|\v x_{≠i}) ~=~ \sum_{\v p∈P}η_{\v p}(\v x_i)⋅\Poly_{\v p}(\v x_{≠i})
\end{align} 
Since $η_{\v p}(\v x)≡x^{p_1}y^{p_2}...z^{p_d}$ are different hence independent monomials, 
$φ_1(\v x_i|\v x_{≠i})$ is invariant under permutations of 
$\v x_{≠i}$ if and only if all $\Poly_{\v p}$ are. 
The latter can all be represented by one large vector function 
consisting of the polynomials $\v x'_i:=ϕ_1(\v x_{≠i})=(\Poly_{\v p}(\v x_{≠i}))_{\v p∈P}$ 
\emph{and} the monomials $\v x''_i:=(η_{\v p}(\v x_i))_{\v p∈P}$.
Since $\v x'_i$ and 
$\v x''_i$ are both output in ``channel'' $i$ of EMLPs, 
any $φ_1(\v x_i|\v x_{≠i})={\v {x''}_i}^\trp \v x'_i$ is a scalar product (polynomial of degree 2) within channel $i$,
hence can be computed by a factored EMLP.
Now by Lemma~\ref{lem:cov}, any equivariant (vector of) polynomials can be computed by an EMLP
as $\v{φ}(\m X)=φ_1(\v x_i|\v x_{≠i})_{i=1}^n$.

\paradot{Equivariant continuous functions}
By Theorem~\ref{thm:fctpoly}, every continuous function can be approximated by a polynomial.
We need a symmetric version of this, which is easy to obtain from the following Lemma:
\begin{lemma}[Symmetric approximation]\label{lem:symappr}
  Let $\cS$ be a finite symmetry group with linear representations on $ℝ^n$ and $ℝ^m$.
  Let $f:ℝ^n→ℝ^m$ be an equivariant function, i.e.\ $S(f(\v x))=f(S(\v x))$ $∀S∈\cS$,
  Let $g:ℝ^n→ℝ^m$ be an arbitrary function and 
  $\overline{g}(\v x):={1\over|\cS|}\sum_{S∈\cS} S^{-1}(g(S(\v x)))$ its linear symmetrization,
  and $||⋅||$ be a norm invariant under $\cS$, then $||f-\overline{g}||≤||f-g||$.
\end{lemma}
The proof is elementary and for permutations and symmetric/equivariant functions with $∞$-norm we care about almost trivial.
The Lemma also holds for compact groups with Haar measure and measurable functions,
e.g.\ for rotations with Euclidean norm, but we do not need this.

\begin{theorem}[Stone-Weierstrass for symmetric/equivariant functions]
   Every continuous function, invariant/equivariant under permutations can be approximated by symmetric/equivariant polynomials,
   uniformly on compacta.
\end{theorem}

\begin{proof}
  Let $g$ be an approximating polynomial of $f$,
  which exists by Theorem~\ref{thm:fctpoly}, 
  and let $\overline{g}$ be its symmetrization.
  If $f$ is invariant/equivariant under permutations,
  then by Lemma~\ref{lem:symappr}, $||f-\overline{g}||_∞≤||f-g||_∞$,
  hence $\overline{g}$ also approximates $f$.
  Since a finite average of polynomials is a polynomial,
  and $\overline{g}$ is symmetric/equivariant by construction,
  this proves the theorem.
\qed\end{proof}

Since EMLPs can approximate all equivariant polynomials and only continuous equivariant functions,
we get one of our main results:

\begin{theorem}[Universality of (two-hidden-layer) EMLP]\label{thm:emlp}
  For any continuous non-linear activation function,
  EMLPs can approximate (uniformly on compacta) all and only the equivariant continuous functions.
  If $σ$ is non-polynomial, a two-hidden-layer EMLP suffices.
\end{theorem}

Indeed, by inspecting our constructive proof, esp.\ \eqref{eq:emlpc} and \eqref{eq:etapoly},
we see that in theory an EMLP with all-but-one layer being factored suffices,
i.e.\ $\m V^\ell=0$ for all-but-one $\ell$. 
In practice we expect an EMLP allowing many/all layers to mix to perform better.
Since 1-hidden-layer MLPs are universal for non-polynomial $σ$ (Thm.\ref{thm:polymlp}\&\ref{thm:fctpoly}),
the factored layers can be merged into 1 layer, leading to a 3-hidden-layer NN,
with first and third layer being factored.
It is easy to see that the second and third layer can actually be merged into one.

\section{(Anti)Symmetric Networks}\label{sec:assnn}

We are finally ready to combine all pieces and define (anti)symmetric NN,
and state and discuss their universality properties.
We briefly remark on why we believe the chosen approach 
is most suitable for deriving interesting error bounds.

\paradot{Universal Symmmetric Network}
We can approximate all and only the symmetric continuous functions
by applying any symmetric continuous function $ς:ℝ^{d'×n}→ℝ^{d'}$
with the property $ς(\v y,...,\v y)=\v y$ to the output of an EMLP, 
e.g.\ $ς(\m Y)={1\over n}\sum_{i=1}^n\v y_i$ or $ς(\m Y)=\max\{y_1,...,y_n\}$ if $d'=1$.
Clearly, the resulting function is symmetric under permutations.
Also, if $ϕ(\m X)$ is any symmetric function, 
then $φ_1(x_i|x_{≠i}):=ϕ(\m X)$ is clearly symmetric in $x_{≠i}$,
hence $\m Y=\v{φ}(\m X):=φ_1(x_i|x_{≠i})_{i=1}^n$ is equivariant and 
can be approximated by an EMLP, 
which, by applying $ς$ to its output $\m Y$, computes $ϕ(\m X)=ς(\m Y)$.
Hence every symmetric continuous function can be approximated
by an EMLP with a final average or max or other symmetric layer.

Note that the detour via EMLP was necessary to construct universal symmetric NNs.
Assume we had started with an MLP for which every layer is a symmetric function,
i.e.\ $\m V^\ell=\m W^\ell$. Such a network could only represent functions of the 
extremely restrictive form $\v y_1=...=\v y_n=ρ_1(\m W\sum_{j=1}^n\v x_j+\v u)$,
where $ρ_1$ is an arbitrary continuous function.

\paradot{Universal AntiSymmmetric Network}
For $d=1$, any continuous AS function $ψ(x_{1:n})$ can be approximated by 
approximating the totally symmetric continuous function 
$χ(x_{1:n}):=ψ(x_{1:n})/\prod_{j<i}(x_i-x_j)$ with an EMLP ($d'=1$),
and then multiply the output by $\prod_{j<i}(x_i-x_j)$.
But we are not limited to this specific construction:
By Theorem~\ref{thm:as} we know that every AS function can be represented
as a GSD of $n$ functions symmetric in all-but-one-variable.
Note that the $φ_i$ in the proof if combined to a vector $\v{φ}$ is \emph{not} equivariant,
but for each $i$ separately, $\v{\tilde{φ}}_i(\m X):=(φ_i(\v x_j|\v x_{≠j}))_{j=1}^n$ is equivariant by Lemma~\ref{lem:cov},
i.e.\ GSD needs an EMLP with $d'=n$.

For $d>1$, we can approximate AS $ψ(\m X)$ by approximating 
the $n$ equivariant $\v{\tilde{φ}}_i$, now with $φ_i$
defined in the proof of Theorem~\ref{thm:asd}, by an EMLP (again $d'=n$),
and then take its Slater determinant.
Note that the $φ_i$ in the proof are defined in terms of a single symmetric
function $ϕ()$, which then gets anti-symmetrized essentially by multiplying with $σ(\bar{π})$.
This shows that an EMLP computing a single symmetric function ($d'=1$) suffices,
but this is \emph{necessarily} and essentially always a discontinuous representation, 
while using the GSD with $d'=n$ equivariant functions 
\emph{possibly} has a continuous representation.

Let us define a (toy) FermiNet as computing the GSD from the output of an EMLP.
The real FermiNet developed in \cite{Pfau:19} contains a number of extra features,
which improves practical performance, theoretically most notably particle pair representations. 
Since it is a superset of our toy definition,
the following theorem also applies to the full FermiNet.
We arrived at the following result:
\begin{theorem}[Universality of the FermiNet]\label{thm:ferminet}
  A FermiNet can approximate any continuous anti-symmetric function.
\end{theorem}
For $d=1$, the approximation is again uniform on compacta.
For $d>1$, the proof of Theorem~\ref{thm:asd} involves discontinuous
$φ_i(\v x_j|\v x_{≠j})$. Any discontinuous function can be approximated
by continuous functions, but not in $∞$-norm, 
but only weaker $p$-norm for $1≤p<∞$. 
This implies the theorem also for $d>1$ in $L^p$ norm.
Whether a stronger $L^∞$ result holds is an important open problem,
important because approximating continuous functions by discontinuous components
can cause all kinds of problems.

\paradot{Approximation accuracy}
The required NN size as a function of approximation accuracy for EMLPs
should be similar to MLPs discussed in Section~\ref{sec:nn} with the following differences:
Due to the permutation (anti)symmetry, weights $\m W,\m V,\v b$ are shared between NN channels $i=1,...,n$,
reducing the number of parameters by a factor of about $n/2$.
On the other hand, the algebraic basis for multisymmetric polynomials has size $({n+d\atop d})≈n^d$,
which is crucially exploited in the polarized power basis,
compared to $n⋅d$ functions suffice for a functional basis.
This of course does not mean that we need a NN of size $O(n^d)$ to accommodate all basis functions,
but if there is a mismatch between the equivariant functions $f$ we care about and the choice of basis,
we may need most of them. 

\section{Discussion}\label{sec:Disc}

\paradoto{Summary}
We reviewed a variety of representations for (anti)symmetric function $(ψ)ϕ:(ℝ^d)^n→ℝ$.
The most direct and natural way is as a sum over $n!$ permutations of some other function $χ$.
If $χ∈\cC^k$ then also $(ψ)ϕ∈\cC^k$. Unfortunately this takes exponential time, or at least is NP hard,
and other direct approaches such as sampling or sorting have their own problem.
The most promising approach is using Equivariant MLPs, 
for which we provided a constructive and complete universality proof,
combined with a trivial symmetrization and 
a non-trivial anti-symmetrization using a large number Slater determinants.
We investigated to which extent a \emph{single generalized} Slater determinant introduced in \cite{Pfau:19},
which can be computed in time $O(n^3)$, can represent all AS $ψ$.
We have shown that for $d=1$, all AS $ψ∈\cC^{k+n(n-1)/2}$ can be represented as $\det Φ$ with $Φ∈\cC^k$.
Whether $Φ∈\cC^k$ suffices to represent all $ψ∈\cC^k$ is unknown for $k>0$. For $k=0$ it suffices.
For $d>1$ and $n>2$, we were only able to show that AS $ψ$ have representations using discontinuous $Φ$.

\paradoto{Open problems}
Important problems regarding smoothness of the representation are open in the AS case.
Whether continuous $Φ$ can represent all continuous $ψ$ is unknown for $d>1$,
and similar for differentiability and for other properties.
Indeed, whether \emph{any} computationally efficient continuous representation of all and only AS $ψ$ is possible is unknown.

\paradot{Outlook: Error bounds}
Our construction via multisymmetric polynomials is arguably more natural,
and can serve as a starting point for interesting error bounds for function classes
that can be represented well by polynomials, 
e.g.\ functions of different degree of smoothness.
Many such results are known for general NN \cite{Pinkus:99}, 
most of them are based on polynomial approximations.
We therefore expect that the techniques transfer to symmetric functions,
with similar bounds, and to AS and $d=1$ with worse bounds due to loss of differentiability. 
For $d>1$ we are lacking smoothness preserving results. 
If and only if they can be established,
we can expect error bounds for this case too. 

\paradot{Acknowledgements}
I want to thank David Pfau for introducing me to the FermiNet and 
Alexander Matthews and James Spencer for providing further references and clarifications.

\bibliographystyle{alpha} 

\begin{thebibliography}{WFE{\etalchar{+}}19}

\bibitem[Act18]{Actor:18}
Jonas Actor.
\newblock {\em Computation for the {{Kolmogorov Superposition Theorem}}}.
\newblock Thesis, May 2018.

\bibitem[Bar93]{Barron:93}
A.R. Barron.
\newblock Universal approximation bounds for superpositions of a sigmoidal
  function.
\newblock {\em IEEE Transactions on Information Theory}, 39(3):930--945, May
  1993.

\bibitem[FMZ19]{Fan:19}
Jianqing Fan, Cong Ma, and Yiqiao Zhong.
\newblock A {{Selective Overview}} of {{Deep Learning}}.
\newblock {\em arXiv:1904.05526 [cs, math, stat]}, April 2019.

\bibitem[GPEB19]{Grohs:19intro}
Philipp Grohs, Dmytro Perekrestenko, Dennis Elbr{\"a}chter, and Helmut
  B{\"o}lcskei.
\newblock Deep {{Neural Network Approximation Theory}}.
\newblock {\em arXiv:1901.02220 [cs, math, stat]}, January 2019.

\bibitem[HLL{\etalchar{+}}19]{Han:19}
Jiequn Han, Yingzhou Li, Lin Lin, Jianfeng Lu, Jiefu Zhang, and Linfeng Zhang.
\newblock Universal approximation of symmetric and anti-symmetric functions.
\newblock {\em arXiv:1912.01765 [physics]}, December 2019.

\bibitem[Kol57]{Kolmogorov:57}
Andrej Kolmogorov.
\newblock On the {{Reprepsentation}} of {{Continuous Functions}} of {{Several
  Variables}} as {{Superpositions}} of {{Continuous Functions}} of {{One
  Variable}} and {{Addition}}.
\newblock 1957.

\bibitem[Liu15]{Liu:15}
Xing Liu.
\newblock Kolmogorov superposition theorem and its applications.
\newblock September 2015.
\newblock Accepted: 2016-04-08T13:42:49Z.

\bibitem[LSYZ20]{Lu:20}
Jianfeng Lu, Zuowei Shen, Haizhao Yang, and Shijun Zhang.
\newblock Deep {{Network Approximation}} for {{Smooth Functions}}.
\newblock {\em arXiv:2001.03040 [cs, math, stat]}, January 2020.

\bibitem[LTR17]{Lin:17}
Henry~W. Lin, Max Tegmark, and David Rolnick.
\newblock Why does deep and cheap learning work so well?
\newblock {\em Journal of Statistical Physics}, 168(6):1223--1247, September
  2017.

\bibitem[Mak96]{Makovoz:96}
Y.~Makovoz.
\newblock Random {{Approximants}} and {{Neural Networks}}.
\newblock {\em Journal of Approximation Theory}, 85(1):98--109, April 1996.

\bibitem[Pin99]{Pinkus:99}
Allan Pinkus.
\newblock Approximation theory of the {{MLP}} model in neural networks.
\newblock {\em Acta Numerica}, 8:143--195, January 1999.

\bibitem[PSMF19]{Pfau:19}
David Pfau, James~S. Spencer, Alexander G. de~G. Matthews, and W.~M.~C.
  Foulkes.
\newblock Ab-{{Initio Solution}} of the {{Many}}-{{Electron Schr\"odinger
  Equation}} with {{Deep Neural Networks}}.
\newblock {\em arXiv:1909.02487 [physics]}, September 2019.

\bibitem[RT18]{Rolnick:18}
David Rolnick and Max Tegmark.
\newblock The power of deeper networks for expressing natural functions.
\newblock In {\em {{ICLR}}}, 2018.

\bibitem[SI19]{Sannai:19}
Akiyoshi Sannai and Masaaki Imaizumi.
\newblock Improved {{Generalization Bound}} of {{Group Invariant}} /
  {{Equivariant Deep Networks}} via {{Quotient Feature Space}}.
\newblock {\em arXiv:1910.06552 [cs, stat]}, 2019.

\bibitem[Wan95]{Wang:95}
Jun Wang.
\newblock Analysis and design of an analog sorting network.
\newblock {\em IEEE Transactions on Neural Networks}, 6(4):962--971, July 1995.

\bibitem[Wey46]{Weyl:46}
Hermann Weyl.
\newblock {\em The Classical Groups: Their Invariants and Representations}.
\newblock Princeton Landmarks in Mathematics and Physics {{Mathematics}}.
  {Princeton University Press}, {Princeton, N.J. Chichester}, 2nd ed., with
  suppl edition, 1946.

\bibitem[WFE{\etalchar{+}}19]{Wagstaff:19}
Edward Wagstaff, Fabian~B. Fuchs, Martin Engelcke, Ingmar Posner, and Michael
  Osborne.
\newblock On the {{Limitations}} of {{Representing Functions}} on {{Sets}}.
\newblock In {\em {{ICML}}}, October 2019.

\bibitem[Yar17]{Yarotsky:17}
Dmitry Yarotsky.
\newblock Error bounds for approximations with deep {{ReLU}} networks.
\newblock {\em Neural Networks}, 94:103--114, 2017.

\bibitem[ZKR{\etalchar{+}}18]{Zaheer:18}
Manzil Zaheer, Satwik Kottur, Siamak Ravanbakhsh, Barnabas Poczos, Ruslan
  Salakhutdinov, and Alexander Smola.
\newblock Deep {{Sets}}.
\newblock In {\em Advances in Neural Information Processing Systems}, pages
  3391--3401, April 2018.

\end{thebibliography}
\begin{small}
\newcommand{\etalchar}[1]{$^{#1}$}

\end{small}

\vfill\pagebreak[3]
\appendix
\section{List of Notation}\label{app:Notation}
\begin{tabbing}
  \hspace{0.16\textwidth} \= \hspace{0.70\textwidth} \= \kill
  \textbf{Symbol}    \> \textbf{Explanation}                                                 \\[0.5ex]
  AS                 \> Anti-Symmetric                                                       \\[0.5ex]
  NN                 \> Neural Network                                                       \\[0.5ex]
  MLP                \> Multi-Layer Perceptron                                               \\[0.5ex]
  EMLP               \> Equivariant Multi-Layer Perceptron                                   \\[0.5ex]
  GSD                \> Generalized Slater Determinant                                       \\[1.5ex]
  $n∈ℕ$              \> number of particles (in physics applications)                        \\[0.5ex]
  $i,j∈\{1:n\}$      \> particle index/number                                                \\[0.5ex]
  $d∈ℕ$              \> dimensionality of particles (in physics applications esp.\ $d=3$)    \\[0.5ex]
  $x∈ℝ$              \> real argument of function or input to NN, particle coordinate        \\[0.5ex]
  $\m x∈ℝ^n$         \> $\m x=(x_1,...,x_n)$, function argument, NN input, $n$ 1d particles ($d=1$) \\[0.5ex]
  $\v x∈ℝ^d$         \> $\v x=(x,y,...,z)^\trp$ vector of coordinates of one $d$-dimensional particle \\[0.5ex]
  $\m X∈ℝ^{d×n}$     \> $\m X=(\v x_1,...,\v x_n)$ matrix of $n$ $d$-dimensional particles   \\[0.5ex]
  $S_n⊆\{1:n\}→\{1:n\}$ \> ~~~~~~~~~~permutation group                                       \\[0.5ex]
  $π∈S_n$            \> permutation of $(1,...,n)$                                           \\[0.5ex]
  $\cS_n⊂ℝ^n→ℝ^n$    \> canonical linear representation of permutation group                 \\[0.5ex]
  $S_π∈\cS_n$        \> $S_π(x_1,...,x_n):=(x_{π(1)},...,x_{π(n)})$                          \\[0.5ex]
  $\cS_n^d≠\cS_{n⋅d}$ \> $d$ copies of linear permutation group representations              \\[0.5ex]
  $S_π^d∈\cS_n^d$    \> $S_π^d(\v x_1,...,\v x_n):=(\v x_{π(1)},...,\v x_{π(n)})$            \\[0.5ex]
  $f:ℝ^{d×n}→ℝ$      \> some function to be approximated by an EMLP                          \\[0.5ex]
  $χ:ℝ^{d×n}→ℝ$      \> general function of $n$ $d$-dimensional particles                    \\[0.5ex]
  $ϕ:ℝ^{d×n}→ℝ$      \> symmetric function: $ϕ(S_π^d(\m X))=ϕ(\m X)$                         \\[0.5ex]
  $ψ:ℝ^{d×n}→ℝ$      \> anti-symmetric (AS) function: $ψ(S_π^d(\m X))=σ(π)ψ(\m X)$           \\[0.5ex]
  $σ(π)=±1$          \> parity or sign of permutation $π$                                    \\[0.5ex]
  $f_\cS$          \> function $f$ linearly symmetrized by $\cS$                           \\[0.5ex]
  $\v x_{≠i}$        \> $≡(\v x_1,...,\v x_{i-1},\v x_{i+1},...,\v x_n)$ all but particle $i$ \\[0.5ex]
  $φ(\v x_i|\v x_{≠i})$ \> function symmetric in $n-1$ arguments $\v x_{≠i}$                 \\[0.5ex]
  $b∈\{1:m\}$         \> index of basis function                                             \\[0.5ex]
  $β_b(\m X)∈ℝ$      \> $b$th basis function, e.g. $\v{β}_b(\m X)=\sum_{i=1}^n\v{η}_b(\v x_i)$ \\[0.5ex]
  $η_b(\v x_j)$      \> $b$th polarized basis function                                       \\[0.5ex]
  $g:ℝ^m→ℝ$          \> composition or outer function, used as $g(\v{β}(\m X))$              \\[0.5ex]
  $ν$                \> some NN/MLP/EMLP function                                            \\[0.5ex]
  $ρ$                \> some (multivariate) polynomial                                       \\[0.5ex]
\end{tabbing}

\end{document}